\documentclass[11pt]{article}

\usepackage{hyperref}

\usepackage{adjustbox}
\usepackage{amsthm}
\usepackage{algorithm}
\usepackage{algpseudocode}
\usepackage{placeins}

\usepackage{enumitem}

\theoremstyle{definition}
\newtheorem{definition}{Definition}[section]
\newtheorem{corollary}{Corollary}[section]
\newtheorem{theorem}{Theorem}[section]
\newtheorem{lemma}{Lemma}[section]
\usepackage{newtxtext,newtxmath}

\usepackage{wrapfig}
\usepackage[utf8]{inputenc} %
\usepackage[T1]{fontenc}    %
\usepackage{hyperref}       %
\usepackage{url}            %
\usepackage{booktabs}       %
\usepackage{amsmath}
\usepackage[capitalize]{cleveref}
\usepackage{amsfonts}       %
\usepackage{nicefrac}       %
\usepackage{microtype}      %
\usepackage{xcolor}         %
\usepackage{subcaption}
\usepackage{graphicx} 
\usepackage{comment} 
\usepackage{etoolbox}
\usepackage{cancel}
\usepackage{soul}
\usepackage{multirow}
\usepackage{float}
\usepackage{wrapfig}
\usepackage{adjustbox}
\usepackage{pifont}
\usepackage{makecell}

\usepackage{mathtools}

\DeclarePairedDelimiterX\braket[2]{\langle}{\rangle}{#1\,\delimsize\vert\,\mathopen{}#2}

\usepackage{tabularx}
\definecolor{lb}{RGB}{31,119,180}
\usepackage{xcolor}
\usepackage{tcolorbox}
\newtcolorbox{mybox}[1]{colback=lb!5!white,colframe=lb!70!black,fonttitle=\bfseries,title=#1}
\usepackage{colortbl}
\tcbuselibrary{skins}
\tcbset{tab2/.style={enhanced,fonttitle=\bfseries,
colback=white,colframe=gray,colbacktitle=white,
coltitle=black,center title}}
\usepackage{pifont}
\newcommand{\cmark}{\textcolor{green!60!black}{\ding{51}}}  %
\newcommand{\xmark}{\textcolor{red!75!black}{\ding{55}}}    %

\date{}
\title{Heat Kernel Goes Topological}

\author{
  Maximilian Krahn \\
  Department of Computer Science,\\ Aalto University, Finland \\
  \texttt{mk2510.github.io}\\ 
  \\
  Vikas Garg \\
  YaiYai Ltd and Aalto University\\
  \texttt{vgarg@csail.mit.edu}\\
}

\begin{document}
\maketitle
\begin{abstract}
Topological neural networks have emerged as powerful successors of graph neural networks. However, they typically involve higher-order message passing, which incurs significant computational expense. We circumvent this issue with a novel topological framework that introduces a Laplacian operator on combinatorial complexes (CCs), enabling efficient computation of heat kernels that serve as node descriptors. Our approach captures multiscale information and enables permutation-equivariant representations, allowing easy integration into modern transformer-based architectures.
Theoretically, the proposed method is maximally expressive because it can distinguish arbitrary non-isomorphic CCs. Empirically, it significantly outperforms existing topological methods in terms of computational efficiency. Besides demonstrating competitive performance with the state-of-the-art descriptors on standard molecular datasets, it exhibits superior capability in distinguishing complex topological structures and avoiding blind spots on topological benchmarks. Overall, this work advances topological deep learning by providing expressive yet scalable representations, thereby opening up exciting avenues for molecular classification and property prediction tasks.

\end{abstract}
\section{Introduction}

The intersection of topology and deep learning has unlocked new possibilities for encoding complex structural information in data. Traditional graph-based neural networks have proven effective in learning representations for structured data \cite{kipf2016semi, maron2019provably}, yet they often struggle with capturing higher-order relationships \cite{besta2024demystifying}. In contrast, combinatorial complexes provide a more expressive framework for modelling such intricate structures, extending beyond pairwise relationships to encode higher-order interactions \cite{hajij2022topological}.
Current topological neural networks often rely on higher-order message passing protocols to capture complex structural relationships \cite{verma2024topological, eitan2024topological}. While effective in theory, these approaches are typically computationally expensive (cf. \cref{sec:experiment}) and offer limited expressive power in distinguishing non-isomorphic structures \cite{eitan2024topological}.
\begin{table}[!t]
\centering
\begin{minipage}{\linewidth}
\centering
\small
    
\begin{tcolorbox}[tab2,tabularx={llccccccll}, boxrule=1pt,top=0.9ex,bottom=0.9ex,colbacktitle=lb!15!white,colframe=lb!70!white]
\begin{minipage}[t]{0.48\textwidth}
\vspace{-10pt}
\centering
\caption*{\textbf{Recent methods for relational data}}
\begin{tabular}{@{}lcccc@{}}
\toprule
\textbf{Method} & \textbf{SC} & \textbf{R2} & \textbf{IS} & \textbf{EQ}  \\
\midrule
CIN  \cite{hajij2020cell}       & \cmark & \xmark & \xmark & \cmark \\
SMCN \cite{eitan2024topological}  & \cmark & \cmark & \xmark & \cmark \\
MCN \cite{eitan2024topological}   & \xmark & \cmark & \cmark & \cmark \\
TopNet \cite{verma2024topological}     & \xmark & \xmark & \xmark &  \cmark \\
iGN    \cite{cai2022convergence}   & \cmark & \xmark & \xmark &  \cmark \\
IEGN      \cite{maron2018invariant}  & \cmark & \xmark & \xmark &  \cmark\\
\midrule
\textbf{TopoHKS (ours)}    & \cmark & \cmark & \cmark & \cmark\\
\bottomrule
\end{tabular}
\end{minipage}
\hfill
\begin{minipage}[t]{0.48\textwidth}
\textbf{Contributions} \\
\smallskip
\textbf{Section 3:}
\begin{itemize}[leftmargin=1.5em]
    \item Defining the Laplacian for CCs
    \item HKS descriptor for CC cells
\end{itemize}

\textbf{Section 4:}
\begin{itemize}[leftmargin=1.5em]
    \item Expressivity of our approach; can distinguish non-isomorphic structures
\end{itemize}

\textbf{Section 5:}
\begin{itemize}[leftmargin=1.5em]
    \item Experiments for expressivity on torus data, scalability, graph and CC classification
\end{itemize}
\end{minipage}
\end{tcolorbox}

    \caption{\textbf{Overview of recent methods for relational data and summary of our contributions}. SC: Scalability, R2: Distinguishability for Rank 2, IS: Isomorphism, EQ: Equivariance.}
    \label{tab:overview}
\end{minipage}
\end{table}
Motivated by these limitations, we propose TopoHKS, a scalable and expressive deep learning framework to distinguish between combinatorial complexes and their isomorphisms. Our method leverages Heat Kernel Signatures (HKS) \cite{sun2009concise} on combinatorial complexes, yielding continuous, multi-scale descriptors that capture the topological neighbourhood of each cell through a diffusion-based perspective. These descriptors encode rich structural information and serve as input to a transformer network, enabling us to model complex interactions without relying on expensive higher-order message-passing mechanisms.
Our contributions can be summarised as follows:

\begin{enumerate}
\item Topological Embeddings for Transformers: We introduce a novel combinatorial complex embedding that is permutation-invariant and computationally efficient, overcoming limitations in existing approaches.
\item Theoretical Expressiveness: We establish our framework's fundamental topological and spectral properties, demonstrating its ability to distinguish non-isomorphic complexes and outperform Weisfeiler-Lehman (WL)-based methods in expressivity.
\item Empirical Performance: Our method achieves state-of-the-art results on molecular prediction benchmarks (MolHIV, Protein) and topological structure datasets, surpassing existing approaches in distinguishing combinatorial complexes.
\item Scalability and Efficiency: Unlike traditional higher-order message passing methods, our approach scales efficiently to large combinatorial complexes.
\end{enumerate}

By bridging topological data analysis and deep learning, our work paves the way for more expressive, computationally efficient, and theoretically grounded approaches to learning on structured data. As shown in \cref{tab:overview}, we formally introduce our framework, establish its theoretical foundations, and validate its performance through extensive experiments. We will publish the code after acceptance. \looseness=-1

\section{Background}

\begin{definition}
    A \textit{combinatorial complex} (CC) is a triple \( (S, \mathcal{X}, \operatorname{rk}) \) consisting of a set \( S \), a subset \( \mathcal{X} \) of \( \mathcal{P}(S) \setminus \{\emptyset\} \), and a function \( \operatorname{rk}: \mathcal{X} \to \mathbb{Z}_{\geq 0} \) with the following properties:
\begin{enumerate}
    \item For all \( s \in S \), \( \{s\} \in \mathcal{X} \), and
    \item The function \( \operatorname{rk} \) is order-preserving, which means that if \( x, y \in \mathcal{X} \) satisfy \( x \subseteq y \), then \( \operatorname{rk}(x) \leq \operatorname{rk}(y) \).
\end{enumerate}

Elements of \( S \) are called \textit{entities} or \textit{vertices}, elements of \( \mathcal{X} \) are called \textit{relations} or \textit{cells}, and \( \operatorname{rk} \) is called the \textit{rank function} of the CC.

\end{definition}

This definition was developed in \cite{hajij2022topological}. When $C$ is a combinatorial complex and we write $C^i$, we mean all cells in $C$ of rank $i$. We can also interpret graphs as combinatorial complexes. Here, graph nodes are cells of rank $0$ and edges are cells of rank $1$.

\begin{definition} (\textbf{Isomorphism})
    Two combinatorial complexes \( C = (S, \mathcal{X}, \operatorname{rk}) \) and \( C' = (S', \mathcal{X}', \operatorname{rk}') \) are said to be \textbf{isomorphic} if there exists a bijective map \( \phi: \mathcal{X} \to \mathcal{X}' \) that preserves both the incidence structure and the rank function. 
\end{definition}

\begin{definition} (\textbf{Higher order incidence matrix})
     The incidence matrix represents the relationships between cells of different ranks. Given a combinatorial complex \( (S, \mathcal{X}, \operatorname{rk}) \), the incidence matrix encodes the boundary relationships between cells of consecutive ranks.  

Formally, let \( \mathcal{X}_k \) denote the set of cells in \( \mathcal{X} \) with rank \( k \). The incidence matrix \( \delta_k \in \mathbb{R} ^{ |\mathcal{X}_{0}| \times |\mathcal{X}_k|} \) has each entry defined as follows:

\[
\delta_k(x, y) =
\begin{cases}
\pm 1, & \text{if } x \subset y,\\
0, & \text{otherwise}.
\end{cases}
\]
More specifically let $x_1, \dots x_n \subset y$ be all cells of $y$, then there exists exactly one $x_i$ s.t. $\delta_k(x_i, y) = 1$ and $\forall x_j $ with $j \neq i$ $\delta_k(x_j,y) = -1$.
This matrix can be interpreted as the \textbf{discrete derivative} operator at rank 0 for combinatorial complexes.
\end{definition}

\paragraph{Hodge Laplacian} On cellular complexes, the Laplacian is known as the \emph{Hodge Laplacian}~\cite{hoppe2024representing}. It is defined as follows. For the next paragraph, let $C$ be a cell complex.  
Let \( \partial_k : C^k \to C^{k-1} \) denote the boundary operator, which maps each \( k \)-cell to its \((k{-}1)\)-dimensional faces. Its transpose, \( \partial_k^\top \), is the \emph{coboundary operator}—also interpretable as the incidence matrix from rank \( k \) to \( k{-}1 \). The \emph{\( k \)-th Hodge Laplacian} is then given by:
\begin{align}    
\Delta_k = \partial_{k+1} \partial_{k+1}^\top + \partial_k^\top \partial_k
\end{align}

This operator acts on \( k \)-cochains (real-valued functions defined on \( k \)-cells) and captures both upward and downward adjacencies, reflecting how cells are connected through lower- and higher-dimensional neighbours. A more detailed definition is provided in \cite{hoppe2024representing}.

While this definition extends naturally to \emph{combinatorial complexes}, their general structure can limit the boundary operators' ability to capture higher-order relationships fully. We illustrate this limitation in \cref{thm:hodge_expressiveness}, showing the restricted expressiveness of classical Laplacians in distinguishing specific structures.

\section{Method}
To address this, we introduce a new approach that defines a Laplacian directly on combinatorial complexes. Next, we construct a heat kernel signature for each node (i.e., a 0-dimensional cell), which serves as the input embedding to our deep learning pipeline. We then describe the whole model setup, while the theoretical and computational benefits of our method are discussed in \cref{Sec:computability}.

\subsection{Laplacian on Combinatorial complexes}

Laplacians are well-defined operators on graphs \cite{hein2007graph} and geometric shapes \cite{ovsjanikov2008global}. However, extending this notion to combinatorial complexes introduces several challenges not present in graphs or simplicial complexes:

\begin{enumerate}
\item \textbf{Single-rank connectivity:} Graphs and manifolds have only one type of connecting element—edges or surfaces—which naturally defines the domain for the Laplacian. In contrast, combinatorial complexes involve connections across multiple ranks, complicating the application of the Laplacian in specific locations and contexts.

\item \textbf{Well-defined hierarchical structure:} In simplicial (chain) complexes, each cell has a fixed rank, and functions are typically defined on cells of a single dimension. The derivative (via boundary or coboundary maps) projects onto adjacent ranks, leading to a natural Hodge Laplacian formulation \cite{forman1998witten}. Combinatorial complexes, however, do not enforce such strict rank stratification, making derivative operations less straightforward to define.
\end{enumerate}

Given these challenges, we aim to design a Laplace operator for combinatorial complexes that is symmetric and positive definite, expressible as a single unified operator, and naturally extends the standard graph Laplacian. Moreover, the operator should meaningfully define the smoothness of functions on the complex and be uniquely determined by the structure and constraints of the problem.

A key aspect of our Laplacian is that it is defined in terms of rank-0 cells, reflecting that downstream tasks typically operate on these base-level elements. This choice ensures that higher-order interactions are captured in their influence on rank-0 cells, with higher-rank cells serving as contextual structures that encode complex dependencies between them.

Given those constraints, we construct a Laplacian on combinatorial complexes and prove its properties.
\begin{definition}\textbf{Combinatorial Complex Laplacian}
\label{def:cc_laplacian}
Let $\mathbf{\delta_i}$ be the incidence matrix from rank $0$ to rank $i$. Furthermore, let $R$ be the maximum rank of the combinatorial complex. Let's define a set $\mathcal{B}$ of size $R$. With the condition that $\forall \mathcal{B'}, \mathcal{B''} \subset \mathcal{B}$ iff $\sum_{b \in \mathcal{B'}}b = \sum_{b \in \mathcal{B}''}b$ than $\mathcal{B'}=\mathcal{B''}$. Meaning all possible subsets-sums are distinct. Such a set can be for instance: $\{2^{-1}, 2^{-2}, \dots, 2 ^ {-n} \}$. We then define the Laplacian with $b_i \in \mathcal{B} \quad \forall i \in \{1,2, \dots, R\}$:
\begin{align}
    \mathbf{L}:= \sum_{i=1}^R  b_i\mathbf{\delta_i} \mathbf{\delta_i}^{\top},
\end{align}
\end{definition}

\subsection{Heat Kernels Descriptors on Combinatorial Complexes}
Given the CC Laplacian definition, we now construct node descriptors based on the HKS algorithm.

\paragraph{Heatkernel on topological structures}
As established in \cite{sun2009concise}, the heat kernel on a compact manifold \( M \) admits the following eigendecomposition:
\begin{equation}
    k_t(x,y) = \sum_{i=0}^{\infty} e^{-\lambda_i t} \phi_i(x) \phi_i(y),
\end{equation}
where \( \lambda_i \) and \( \phi_i \) are the \( i \)th eigenvalue and the \( i \)th eigenfunction of the Laplace–Beltrami operator, respectively.

The difference is that in this case, we use the Laplacian on Combinatorial Complexes instead of the Laplace-Beltrami operator. Further, we only consider the kernel for cells of rank 1.
This creates the following kernel matrix:
\begin{equation}
    K_t = \exp(-tL),
\end{equation}
\begin{wrapfigure}[10]{r}{0.35\linewidth}
\vspace{-12pt}
    \centering
    \includegraphics[width=\linewidth]{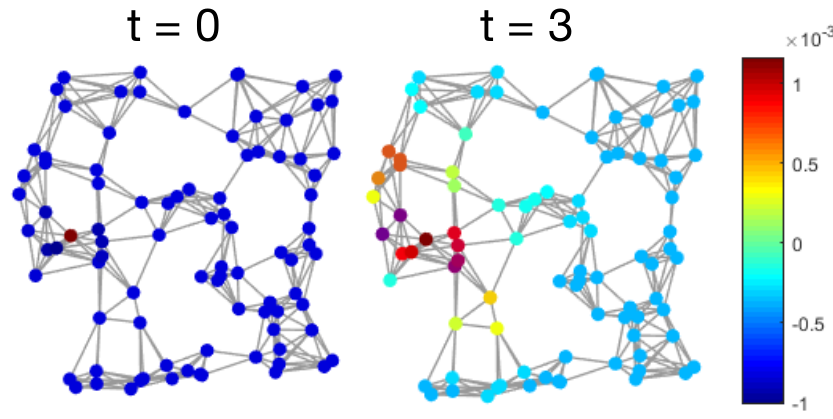}
    \caption{\small Qualitative example of how the Heat Kernel captures local and global information. %
    }
    \label{fig:HeatDiffusion}
\end{wrapfigure}
\begin{wrapfigure}[13]{r}{0.4\linewidth}
\vspace{-10pt}
    \centering
    \includegraphics[width=\linewidth]{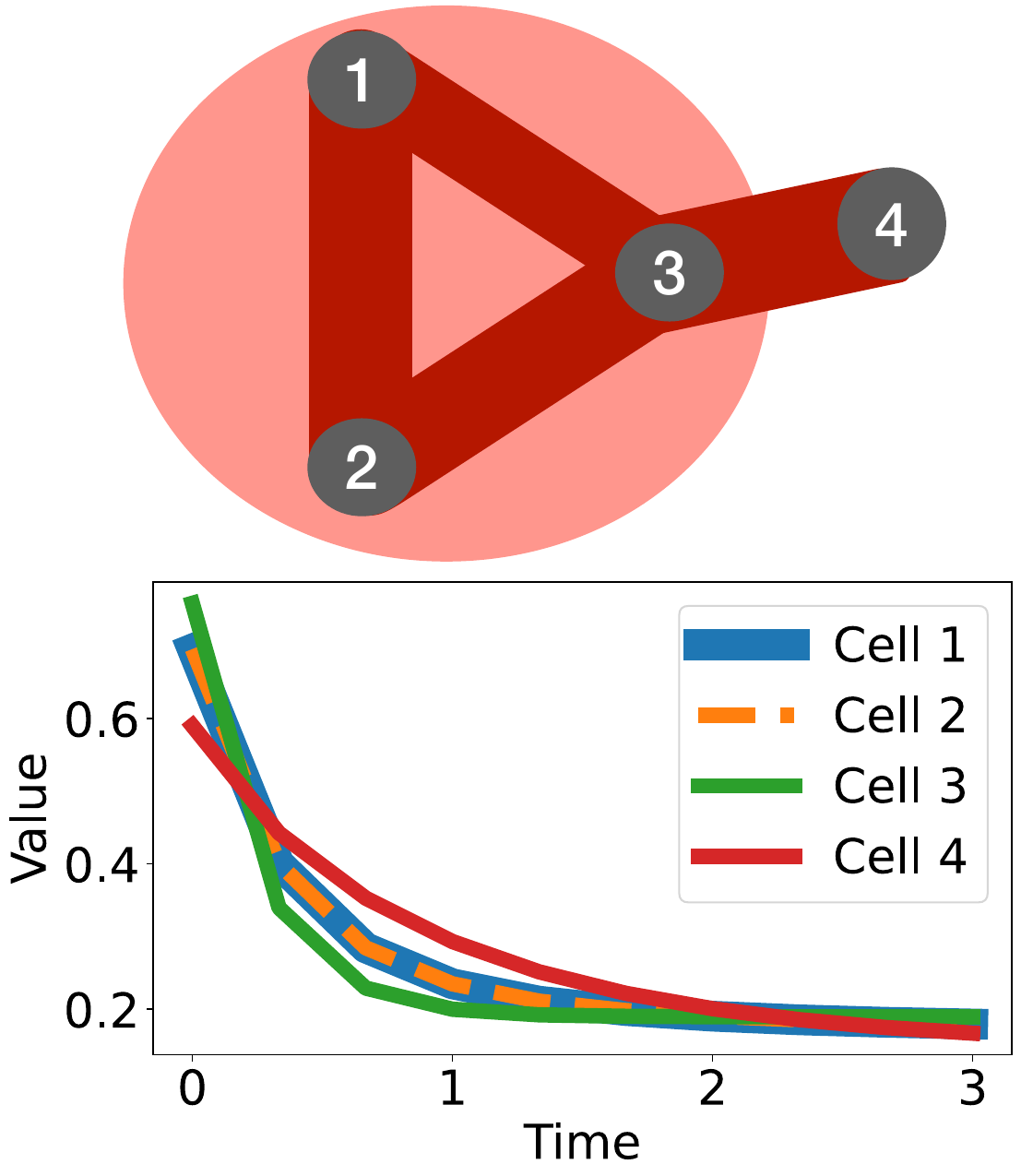}
    \caption{\small Qualitative example of how the HKS descriptor differs for non-isometric cells of rank 0. %
    }
    \label{fig:HKSSIG}
\end{wrapfigure}

Let \( t \) denote the diffusion time and \( L \) the Laplacian of the combinatorial complex. For efficient computation, we use the spectral decomposition of \( L \), given by  
\(
L = \Phi^\top \mathrm{diag}(\lambda_1, \dots, \lambda_n) \Phi,
\)
which allows us to express the heat kernel as  
\(
K_t = \Phi^\top \mathrm{diag}(e^{-t\lambda_1}, \dots, e^{-t\lambda_n}) \Phi.
\)
Let \( c \) be a rank-0 cell in the combinatorial complex. For notational convenience, we define \( K_t(c) := e_c^\top K_t e_c \), where \( e_c \in \mathbb{R}^n \) is the standard basis vector corresponding to cell \( c \).

As illustrated in \cref{fig:HeatDiffusion}, the kernel captures increasingly global structural information as \( t \) increases, reflecting local neighbourhoods at small \( t \) and global connectivity at large \( t \). Motivated by this behaviour, we define the descriptor as a multi-scale embedding based on diffusion at time points \( t_1, \dots, t_d \).

\begin{definition} \textbf{Heat Kernel descriptor}
\label{def:hks_cc}
The \emph{descriptor} of a cell \( c \) of rank $0$ in a combinatorial complex \( C \) is a vector in \( \mathbb{R}^d \). Further we have $d$ times named $t_1, \dots, t_d$. Let $K_t$ be the heat kernel matrix with the variable time parameter $t$. The descriptor is defined as:
\begin{align}    
\mathrm{HKS}_{t_1, \dots, t_d}(c) = [K_{t_1}(c), \dots, K_{t_d}(c)],
\end{align}
\end{definition}
In \cref{fig:HKSSIG}, it is shown how the descriptors capture the different topological neighbourhoods of the four cells of rank 0.

\subsection{Training}

The training pipeline aims to learn a single feature vector representing the entire combinatorial complex. We first compute the Heat Kernel Signature (HKS) for all rank-0 cells, as defined in \cref{def:hks_cc}, and concatenate it with their existing features to form enriched input representations. These are enhanced with positional encodings and passed through a linear embedding layer to align with the transformer architecture.

Let the input be \( \mathbf{X} \in \mathbb{R}^{B \times N \times D} \), where \( B \) is the batch size, \( N \) the number of rank-0 cells, and \( D \) the feature dimension. Let \( \mathbf{G} \in \mathbb{R}^{D \times E} \) be a basis matrix with \( E \) the basis dimension. The embedding $\mathbf{O}$ is computed as:
$\mathbf{O} = \begin{bmatrix} \sin(\mathbf{X} \cdot \mathbf{G}) & \cos(\mathbf{X} \cdot \mathbf{G}) \end{bmatrix}
$
These encoded features are processed by \( n \) self-attention layers, followed by a multi-layer perceptron (MLP) to produce the global feature vector. An overview of the architecture is shown in \cref{fig:training_pipeline}.

\paragraph{MLP Mixer Backbone}
Furthermore, we propose using the MLP Mixer as our learning backbone. Unlike transformer architectures that rely heavily on attention mechanisms, the MLP Mixer processes spatial information through alternating token-mixing and channel-mixing MLPs. This approach maintains computational efficiency while effectively capturing token-level interactions and feature representations. The token-mixing MLPs operate across spatial dimensions (treating each token as a channel), while channel-mixing MLPs process feature dimensions independently per token. This decomposition enables our model to learn spatial relationships without the quadratic complexity of self-attention. In our experiments, the MLP Mixer demonstrates comparable or superior performance to transformer-based approaches, particularly for tasks where global feature interactions are crucial, requiring fewer computational resources and exhibiting faster convergence during training.

\begin{figure}
    \centering
    \includegraphics[width=0.8\linewidth]{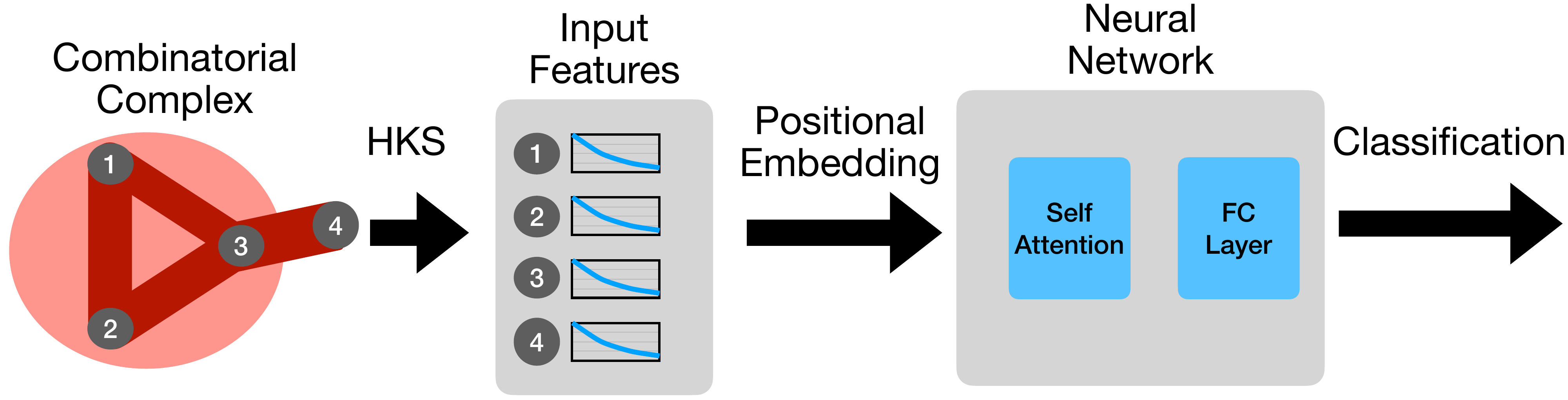}
    \caption{Showing a model setup for learning features. Input: Concatenation of cell features of rank $0$ and the calculated HKS Descriptors. Model: Transformer with positional encoding and $N$ self-attention layers. 
    }
    \label{fig:training_pipeline}
\end{figure}

\section{Theoretical Results}
To support our architectural design and choice of HKS descriptors, we now turn to the theoretical foundations of our framework. We first show that the proposed Laplacian operator satisfies the properties expected of a valid Laplacian and draw connections to the classical Hodge Laplacian. We then analyse the expressive power of our method, demonstrating its ability to capture complex structural information inherent in combinatorial complexes.
This section outlines key proofs, while complete derivations are presented in the Appendix.

\subsection{Laplacian Properties}
Graphs are a special case of  CCs with only rank-0 and rank-1 cells, where each rank-1 cell connects exactly two rank-0 cells. In such cases, the CC Laplacian should reduce to the standard graph Laplacian. \looseness=-1
\begin{corollary} (\textbf{Relationship between CC and Graph Laplacians})
For graphs, combinatorial complexes with rank 1, the Laplacian from \cref{def:cc_laplacian} is identical to the Laplacian defined on graphs.
\end{corollary}

We further had the uniqueness requirement for CC Laplacians, which we show with the following theorem.
\begin{theorem}(\textbf{Uniqueness of Laplacians for CCs})
\label{them:uniqueness_of_laplacians}
Let \( L \) be the Laplacian of a combinatorial complex \( C \). The Laplacian of a combinatorial complex is uniquely determined. Moreover, if there exists an invertible orthogonal matrix \( \mathbf{\Pi} \in \mathbb{R}^{n \times n} \) such that  
\begin{align}
      L' = \mathbf{\Pi} L \mathbf{\Pi}^{\top},
\end{align}
  
then \( L' \) serves as the Laplacian of another combinatorial complex \( C' \), which is spectrally equivalent to \( C \), meaning there exists a bijective, unique mapping between the two combinatorial complexes, which makes them isomorphic.
\end{theorem}
The proof closely follows the following line of thought.
We argue it by constructing the Laplacian. The Laplacian is uniquely built as a weighted sum of the individual components. Hence, the Laplacian will be unique if all components are unique for the underlying combinatorial complex.
As the individual components are simply variations of adjacency matrices, and adjacency matrices are unique in terms of the underlying connected structure, we can argue that our Laplacian is unique to the combinatorial complex.
However, this uniqueness is based on the same principle as that of adjacency matrices.
Hence, we showed that each Laplacian has a unique spectrum.

\begin{wrapfigure}[14]{R}{0.45\textwidth}
        \vspace{-20pt}
        \includegraphics[width=\linewidth]{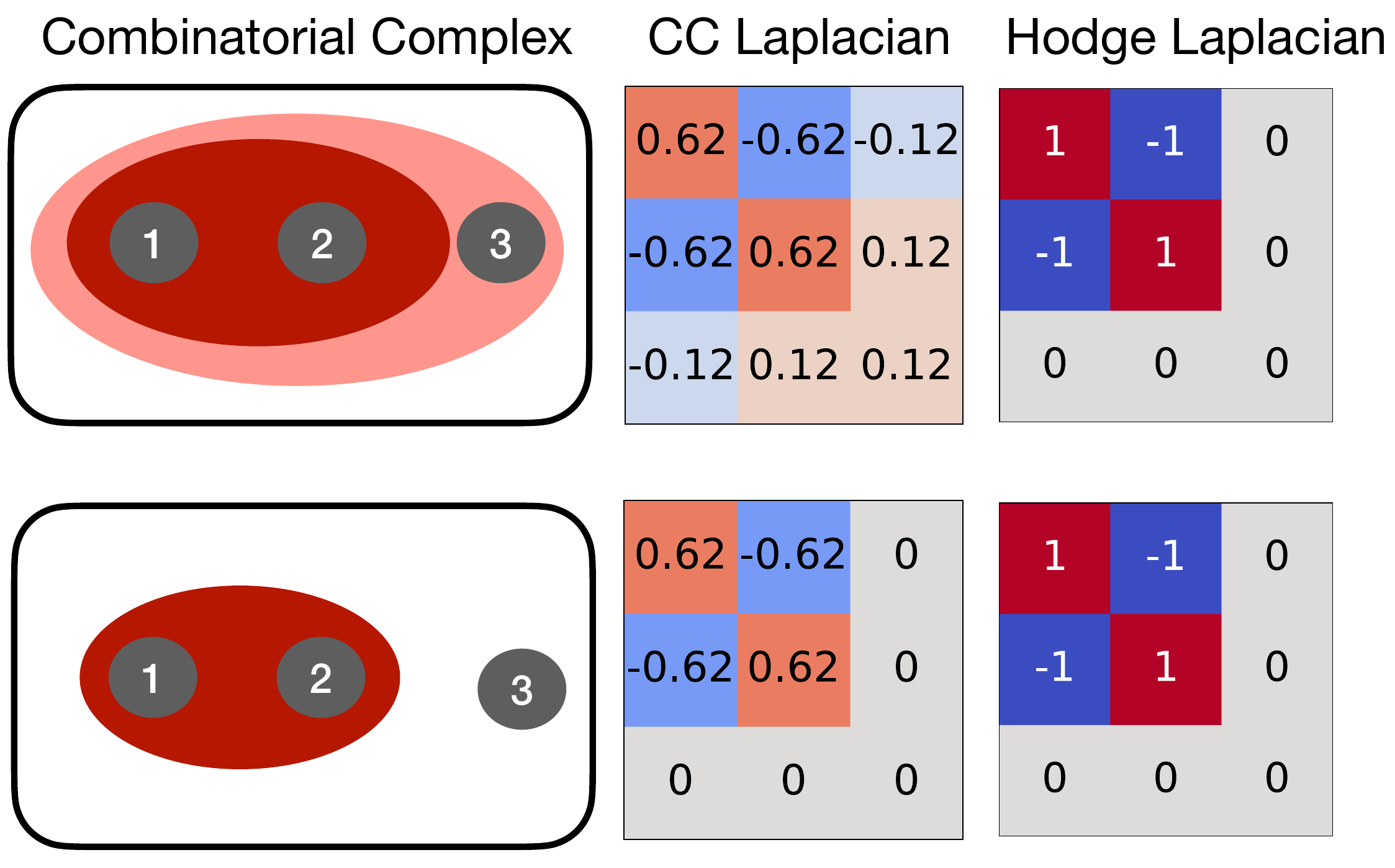}
    \caption{Presenting two Combinatorial Complexes with their CC and Hodge Laplacian. While the CC Laplacian differs, the Hodge Laplacian is the same for both complexes.}
    \label{fig:counter_example_hodge}
\end{wrapfigure}
\begin{lemma}(\textbf{Non-uniqueness of Hodge Laplacians on CC}) Hodge Laplacians on Combinatorial complexes are not unique, meaning there exists a pair of Combinatorial Complexes which share a Hodge Laplacian, but are not isomorphic
\end{lemma}

This corollary can be demonstrated by constructing a simple counterexample, as illustrated in \cref{fig:counter_example_hodge}. The two combinatorial complexes shown there are non-isomorphic. We observe that cells 1, 2, and 3 are of rank 0. Cells 1 and 2 are connected to a rank-1 cell in both complexes. However, in the first complex, an additional rank-4 cell connects cell 3 to the rest of the structure through a higher-order relationship. This higher-order interaction renders the two complexes non-isomorphic. The CC Laplacian can capture this difference, as it accounts for all higher-rank interactions, including the presence of the rank-4 cell. In contrast, the Hodge Laplacian fails to detect this distinction, as the rank-4 cell is not connected to any rank-3 or rank-5 cell, and thus contributes zero under the boundary and Laplacian operators.

As a result, the Hodge Laplacians of the two complexes are identical, highlighting their non-uniqueness. In contrast, the CC Laplacian defined in \cref{def:cc_laplacian} successfully distinguishes them, offering a more expressive and discriminative representation.

\begin{corollary}
\label{thm:hodge_expressiveness}
(\textbf{Hodge Laplacian Expressiveness})    On Combinatorial Complexes, the Laplacian in \cref{def:cc_laplacian} is strictly more expressive than the Hodge Laplacian, and on Simplicial Complexes they are equally expressive.
\end{corollary}
While the first part follows directly from the previous lemma, the second statement relies on the structural property of cell complexes. Each cell is composed exclusively of cells of one rank lower. As a result, the Hodge Laplacian captures all valid interactions between cells in this setting.

\paragraph{Laplacian interpretations} 
\begin{theorem}
(\textbf{Smoothness})
    Let $L$ be the Laplacian descriptor for the Combinatorial Complex $C$. We then define a function $f: C^0 \rightarrow \mathbb{R}$. Based on this:
\begin{align}
    f^\top L f
    \label{eq:smoothness}
\end{align}
expresses the smoothness of a function defined on rank-0 cells of the combinatorial complex. In particular, \cref{eq:smoothness} can be reformulated as 
\begin{align}
    B\sum_{i,j} w_{i,j}(f_i - f_j)^2~,
    \label{eq:smoothness_itemised}
\end{align}
\end{theorem}
where \( B = \sum_{i=0}^R \beta_i \)
The connection between \cref{eq:smoothness} and \cref{eq:smoothness_itemised} is easily shown by using the \cref{def:cc_laplacian} of the Laplacian and writing out the inner product.
We provide more details in the Appendix.

It is important to note that in \cref{eq:smoothness_itemised} $w_{ij}$ is $0$, iff there is no cell of any rank connecting the two cells of $i$ and $j$.
With the formulation of \cref{eq:smoothness_itemised}, it becomes clear that \cref{eq:smoothness} is a good measure for smoothness. 
This form highlights its role as a discrete Dirichlet energy, quantifying the extent to which $f$ varies across connected cells. Small values of $f^\top L f$ indicate that $f$ changes gradually along connections, implying a smooth signal over the Combinatorial Complex.

\paragraph{Observation} Smoothness of $L$: When interpreting $L$ as a function of the rank, namely: $L(r) := \sum_{i=0}^r b_i\delta_i \delta_i^\top$, we observe that $L$ does not change much for higher rank, as their individual contributing weights become smaller.

\subsection{Expressiveness of our approach}
\label{Sec:computability}
In this section, we establish the computational expressiveness of our proposed method. First, we show that non-isomorphic Combinatorial Complexes will have different HKS descriptors.
Then, based on this result, we show the expressiveness of our proposed method.

\paragraph{A universal function approximator} is a neural network, which structurally can learn to approximate any function \cite{hornik1989multilayer}. Further, it has been shown that Transformers can also approximate functions \cite{yun2019transformers}.

\begin{theorem} 
\textbf{(HKS uniqueness)}
    Let \( L \) and \( L' \) be two Laplacians such that \( L' \neq \mathbf{\Pi} L \mathbf{\Pi}^{\top} \) for any orthogonal matrix \( \mathbf{\Pi} \). Then the corresponding Heat Kernel Signature (HKS) descriptors derived from \( L \) and \( L' \) are distinct.
\end{theorem}
The proof follows two steps.
\begin{enumerate}
    \item \textbf{Uniqueness of the Laplacian:} We first show that if two combinatorial complexes \( C \) and \( C' \) are not isomorphic—i.e., there exists no bijective map between their cells that preserves the incidence structure—then their Laplacians \( L \) and \( L' \) are not similar. That is, there exists no invertible matrix \( \mathbf{\Pi} \) such that \( L' = \mathbf{\Pi} L \mathbf{\Pi}^\top \).
    
    \item \textbf{Diffusion Distinguishability:} Given that the Laplacians are not similar, we then show that their corresponding diffusion patterns (e.g., heat kernels or heat kernel signatures) must differ. This implies that the descriptors derived from diffusion processes can effectively distinguish between non-isomorphic complexes.
\end{enumerate}

Having now proven that different CC have different HKS descriptors, we can now show that our learning approach can distinguish any non-isomorphic combinatorial complexes.

\begin{corollary} (\textbf{Expressiveness})
Given two combinatorial complexes with distinct input descriptors, it is possible to learn a function using a Universal Function Approximator (UFA) approach that effectively distinguishes between them. This means we can determine any WL classes by theoretical design. \looseness=-1
\end{corollary}

We argue that if the neural network's input is distinctive, we can learn a function with a unique output with the UFA. This makes our method able to distinguish between CC up to isomorphism.

\section{Experiments}
\label{sec:experiment}
This section demonstrates that we can outperform the SMCN method on combinatorial complexes, which differ in cells of rank at least 3. Furthermore, we can distinguish between any combinatorial complexes as SMCN can, which are listed in the torus dataset.
Afterwards, we will also test our method on established benchmarks.
During the evaluation, we will demonstrate that 1) our method is more expressive than the SMCN and 2) is on par with other baselines for the real-world datasets.
All models were trained on a single NVIDIA V100 GPU with 32 GB of memory. Training took approximately 1 to 5 hours, depending on the dataset. This project consumed a total of 6,000 GPU hours. In our implementation, we typically use $d=10$ and an equal spacing between $0$ and $3$ for the times.

\subsection{Tori and higher order combinatorial complexes}
\begin{wraptable}{r}{0.5\linewidth}
\vspace{-10pt}
\centering
\begin{adjustbox}{width=\linewidth}
\begin{tabular}{lccc}
\hline Model &  \# & Accuracy & Speed \\
\hline SMCN  \cite{eitan2024topological}& 223 & $100 \%$ & $7$ it/s\\
CIN \cite{hajij2020cell} & 0 & $0 \%$ & $10$ it/s \\
TopoHKS (ours) & 223 & $100\%$ & $95$ it/s\\
\hline
\end{tabular}
\end{adjustbox}
\caption{Topological Blind Spot torus dataset}
\label{tab:topBlindTorus}
\vspace{-20pt}
\end{wraptable}

We start by comparing our method to the SMCN and HOMP for the torus dataset. This dataset is constructed analogously as mentioned in \cite{eitan2024topological}.
\cref{tab:topBlindTorus} compares our method with the SMCN and CIN.
When testing the method, we notice that, as shown in the proof section, our process is at least as expressive as the SMCN method and also more expressive than CIN.

We created a new dataset to prove that we also outperform SMCN in terms of expressivity. 
This dataset consists of a pair of tori, similar to the original topological blind spot dataset.
The only difference is that the tori are different by one cell of rank $4$, which covers two cells of rank $2$.

\begin{figure}%
    \centering
    \vspace{-10pt}
    \begin{subfigure}[b]{0.6\linewidth}
    \begin{adjustbox}{width=\textwidth}
        \begin{tabular}{lccc}
\hline Model & Distinguished Pairs & Accuracy & Speed \\
\hline SMCN \cite{eitan2024topological}& 0 & $0 \%$ & $10$ it/s\\
CIN \cite{hajij2020cell}& 0 & $0 \%$ & $9$ it/s \\
TopoHKS (ours) & 223 & $100\%$ & $100$ it/s\\
\hline
\end{tabular}
\end{adjustbox}
\caption{\small Table with accuracies and speed on modified torus dataset}
\label{tab:table_modified_torus}
    \end{subfigure}
    \hfill
    \begin{subfigure}[b]{0.19\linewidth}
        \includegraphics[width=\linewidth]{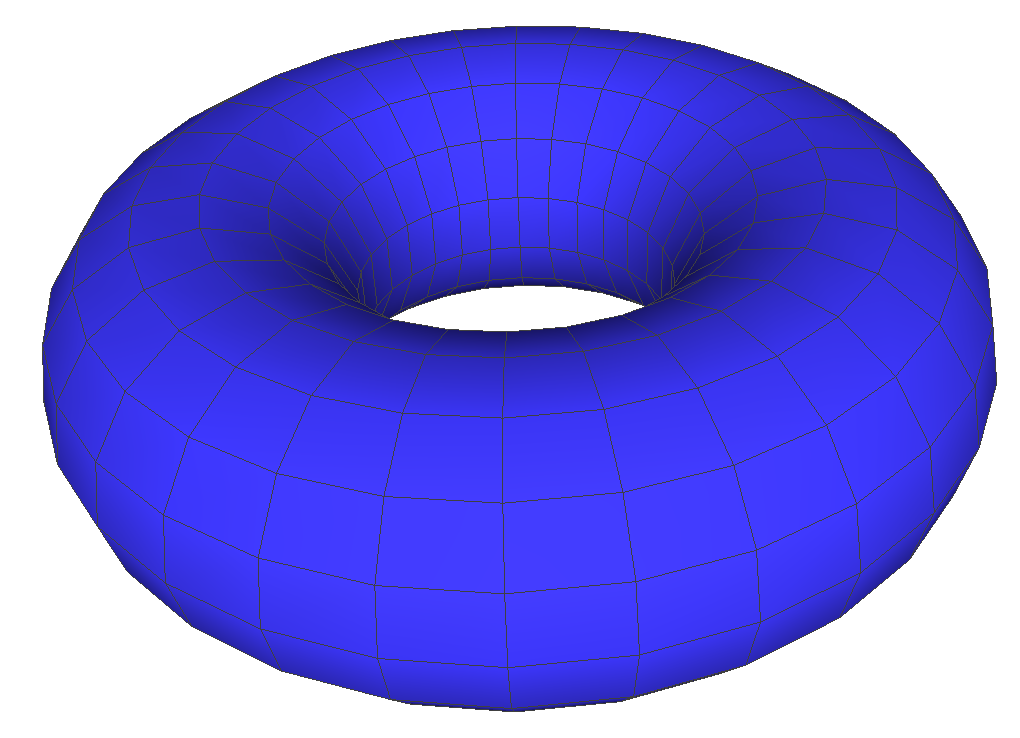}
        \caption{Torus with only cells of rank 0,1,2}
        \label{fig:torus_cell_1_2}
    \end{subfigure}
    \hfill
    \begin{subfigure}[b]{0.19\linewidth}
        \includegraphics[width=\linewidth]{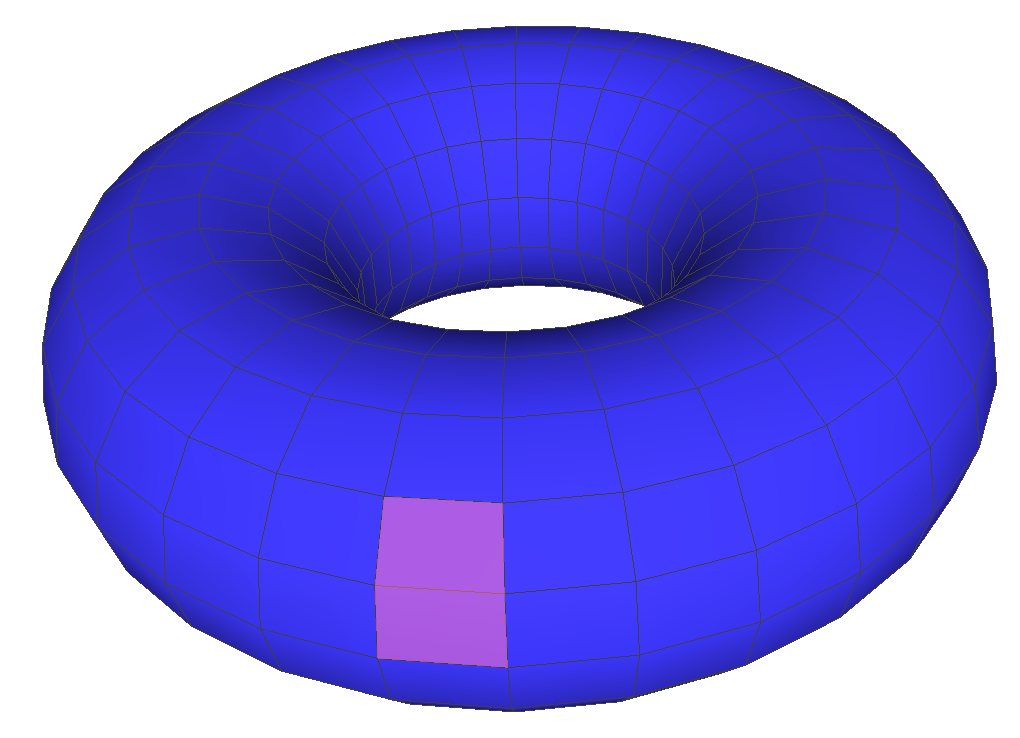}
        \caption{Torus with cells of rank 0,1,2,3}
        \label{fig:torus_cell_1_2_3}
    \end{subfigure}
    \caption{In \cref{tab:table_modified_torus} we show that only our method can differentiate between tori with a cell of rank 3 differing. In \cref{fig:torus_cell_1_2} and \cref{fig:torus_cell_1_2_3} visualise an example datapoint, which can be distinguished by our method, but not by SMCN}
    \label{fig:higherOrderDifferent}
\end{figure}

In \cref{fig:higherOrderDifferent} we observe what we also assumed from the proof section.
Our proposed method can distinguish between combinatorial complexes, which differ in cells of higher rank.
This example also demonstrates that our method can distinguish pairs of combinatorial complexes down to isomorphism. An important note is that this only works with the CC Laplacian and not with the Hodge Laplacian.

\subsection{Scalability comparison}

\begin{wrapfigure}[11]{r}{0.3\textwidth}
\vspace{-10pt}
    \includegraphics[width=\linewidth]{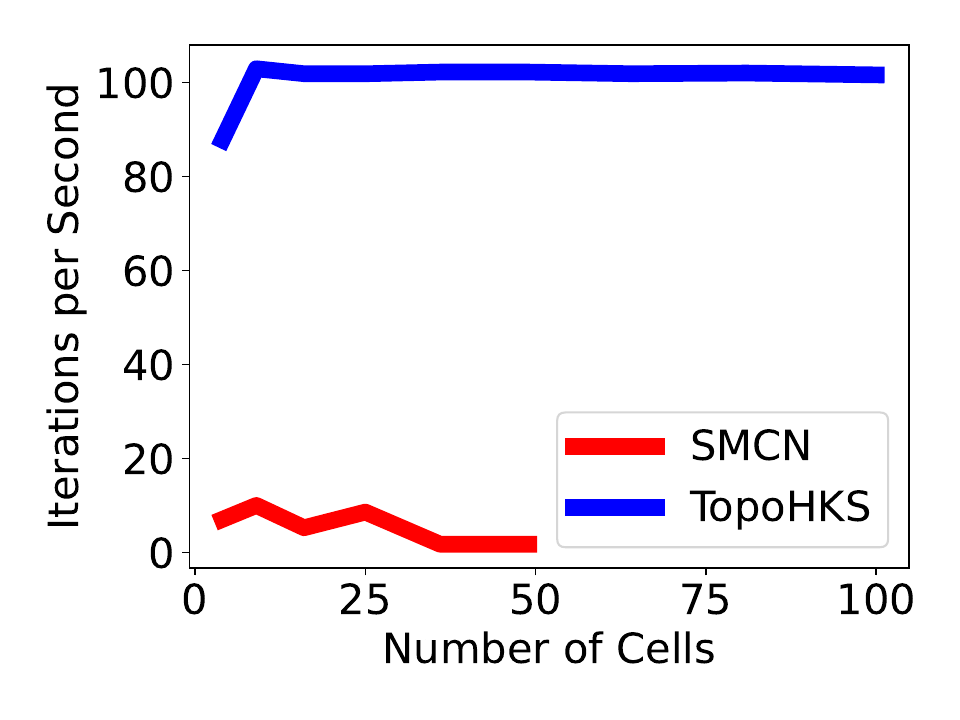}
    \caption{Inference timing for differently sized combinatorial complexes. Mean over five runs %
    }
    \label{fig:inference_timing}
    \vspace{-40pt}
\end{wrapfigure}

We evaluate the scalability of our method against the Scalable-MCN (SMCN) model from \cite{eitan2024topological}, using a modified Torus dataset. Instead of focusing on classification accuracy, we concentrate on computational performance as the number of cells increases from 4 to 100. As shown in \cref{fig:inference_timing}, SMCN fails to scale beyond 50 cells due to hardware limitations (GPU V100, 32GB). In contrast, our method maintains a constant computational footprint regardless of complex size and benefits from an efficient transformer backbone, achieving higher GPU throughput. On average, it is approximately 12 times faster than SMCN. \looseness=-1

\subsection{Graph Classification - Benchmarks}

\begin{table}[t]
  \centering
  \adjustbox{max width=\textwidth}{
\begin{tabular}{lcc|cc || cc | cc}
\hline Model & \multicolumn{2}{c}{MolHIV}  & \multicolumn{2}{c}{PROTEIN} &\multicolumn{2}{c}{Glycose} & \multicolumn{2}{c}{Immuginicity}\\
 &  ACC & speed & ACC & speed & MCC & speed & MCC & speed \\
\hline
GCN \cite{kipf2016semi} & $76.06 \pm 0.97$ & $0.04$s & $75.53\pm1.62$ & $0.003$s & NaN & NaN & $0.78 \pm 0.02$ & $20.00$ it/s \\
GIN \cite{xu2018powerful} & $75.58 \pm1.40$ & $0.007$s & $75.54\pm1.85$ & $0.007$s & $ 0.89 \pm 0.02$ & $42$ it/s & $0.80 \pm 0.02$ &$44.14$ it/s \\
 
\hline\hline SMCN \cite{eitan2024topological} & $81.16 \pm 0.90$ & $3.1$s &$72.8 \pm 1.5$& $2.8$s& $0.90 \pm 0.07$ & $3.2$ it/s & $0.85 \pm 0.01$ & $3.3$ it/s \\
CIN \cite{hajij2020cell} & $80.94 \pm 0.57$ & $3.0$s &$77 \pm 4.2$ &$2.9$s & NaN & NaN & NaN & NaN \\
TopNetX \cite{verma2024topological} & $75.98 \pm1.80$ &$8.4$s & $73.79 \pm 1.45$& $5.3$s &$0.71 \pm 0.02$&$0.6$ it/s&$0.62 \pm 0.04$ & $0.5$ it/s\\
TopoHKS (ours) & $\mathbf{83.1 \pm  1.3} $ & $\mathbf{0.65}$s& $\mathbf{79.2 \pm 1.0}$& $\mathbf{0.3}$s & $\mathbf{0.90 \pm 0.03}$ & $\mathbf{ 4.1}$it/s & $\mathbf{0.84 \pm 0.01}$ & $\mathbf{4.0}$it/s\\
\hline
\end{tabular}
}
\caption{Performance on Graph Classification (MolHIV and PROTEIN) and simplicial complex datasets (Glycose, Immogenicity) CIN failed to learn on Glycose and Immogenicity}
\label{tab:nat_datasets}
\vspace{-10pt}
\end{table}  
 
Next, we compare our method on established graph classification benchmarks such as MolHIV and PROTEIN. The train/val/test setup is equivalent to \cite{verma2024topological}. As shown in \cref{tab:nat_datasets}, our method outperforms the state-of-the-art topological methods and graph neural networks.
While we are not faster than the lightweight GNNs, our inference time is only a fraction of that of the established topological methods.
This further demonstrates that with descriptive features, transformers are suitable network architectures for topological learning.

\subsection{GIFFLAR with MLP Mixer}

The Gifflar dataset \cite{joeres2024higher} includes naturally higher-order connections and consists of classification tasks: \emph{Glycose} (3 classes) and \emph{Immunogenicity} (2 classes), evaluated using MCC. We train all models using the train/val/test split from \cite{joeres2024higher}, with results shown in \cref{tab:nat_datasets}.
Our method matches the performance of SMCN while achieving significantly faster inference. It also outperforms both graph neural networks and TopNetX. Notably, CIN fails to learn from these datasets, highlighting the need for expressive yet efficient higher-order neural networks.

\section{Related Work
}

\paragraph{Node Embedding}

\textit{Spectral embeddings} are based on the graph Laplacian, defined as $L = D - A$, where $D$ is the degree matrix and $A$ the adjacency matrix \cite{streicher2023graph}. These methods leverage the Laplacian’s eigenstructure to capture structural properties of graphs. Chung's foundational work \cite{chung1997spectral} established the theoretical basis for spectral analysis, leading to techniques such as spectral clustering \cite{ng2001spectral} and Laplacian Eigenmaps \cite{belkin2003laplacian}, which embed nodes into low-dimensional spaces while preserving local neighbourhoods. Diffusion Maps \cite{coifman2006diffusion} extended this idea to capture multi-scale connectivity. More recently, spectral methods have been applied in graph neural networks to learn node features through message passing \cite{kipf2016semi, fu2022p, runwal2022robust}. \looseness=-1

\paragraph{Topological Deep Learning}
Topological Deep Learning (TDL) enables learning from data on topological structures, with research primarily focused on hypergraphs, simplicial, cellular, and combinatorial complexes. Higher-order message passing (HOMP), developed initially for simplicial complexes \cite{bodnar2021weisfeiler2, battiloro2024n, roddenberry2021principled, yang2023convolutional, ebli2020simplicial, goh2022simplicial, battiloro2024generalized, maggs2023simplicial, lecha2025higher}, has since been extended to cellular \cite{bodnar2021weisfeiler, hajij2020cell} and combinatorial complexes \cite{hajij2022higher, hajij2022topological}, achieving strong empirical performance and enhancing the expressiveness of MPNNs. In 
\cite{papillon2023architectures}, the message passing approaches have been summarised. Complementary approaches integrate pre-computed topological features into MPNNs \cite{horn2021topological, chen2021topological, ballester2023expressivity} and HOMP models \cite{verma2024topological, buffelli2024cliqueph}, further boosting performance and underscoring the value of topological priors. Recent efforts include a standardisation of topological approaches by unifying the benchmarks \cite{telyatnikov2024topobenchmarkx} and the model architecture \cite{papillon2024topotune}.

\paragraph{Expressivity}

The expressivity of GNNs is often assessed through their separation power or ability to distinguish non-isomorphic graphs. Seminal works  \cite{morris2019weisfeiler, xu2018powerful} established that MPNNS are as expressive as the 1-WL test \cite{weisfeiler1968reduction}. General \emph{isomorphism} would be in an infinite WL class. This limitation led to the development of more expressive GNNS, surpassing 1-WL but often at higher computational costs. Notably, \cite{morris2019weisfeiler, maron2018invariant} introduced architectures matching the k-WL test with $\mathcal{O}(n^k)$ complexity. Other approaches enhance expressivity through random features \cite{abboud2020surprising}, substructure counts \cite{bouritsas2022improving}, equivariant polynomials \cite{maron2019provably, puny2023expressive}, and subgraph processing 
\cite{bevilacqua2021equivariant, frasca2022understanding, zhang2023complete,zhang2021nested,cotta2021reconstruction}. Expressive topological models have been built in \cite{maron2019provably, zhang2023complete,bar2024flexible, bamberger2022topological}.
However, the authors in \cite{eitan2024topological} showed that simple HOMP can't distinguish between combinatorial complexes that have the same cover. They introduced \emph{SMCN} \cite{eitan2024topological} to address these expressive issues.
However, SMCN can only distinguish CC up to rank $3$.
For an extensive review of expressive GNN architectures, see \cite{jegelka2022theory, morris2023weisfeiler,zhang2024expressive}.

\paragraph{Heat Kernel Signatures} (HKS) \cite{sun2009concise} are diffusion-based descriptors originally developed for shape analysis, capturing intrinsic geometry through the eigenvalues and eigenfunctions of the Laplace-Beltrami operator. Subsequent works refined this idea \cite{bronstein2010scale, raviv2010volumetric, ovsjanikov2010one, liu2022novel}, while related descriptors, such as the Wave Kernel Signature (WKS) \cite{aubry2011wave}, emphasise different spectral and temporal properties. HKS has also been applied to graphs \cite{donnat2018learning}, but its potential for improving GNN expressivity remains unexplored. \looseness=-1

\section{Conclusion}

We introduced a framework that integrates the Heat Kernel Signature (HKS) with combinatorial complexes to build expressive, \textbf{permutation-invariant} representations for deep learning. We computed multi-scale heat descriptors as robust alternatives to traditional embeddings by defining a Laplacian on combinatorial complexes. Our method proved both theoretically expressive, in that it distinguishes non-isomorphic complexes, and empirically strong, achieving state-of-the-art results on molecular and topological benchmarks. It also scales efficiently, outperforming existing methods in runtime. These results highlight the value of topological descriptors in enhancing graph and complex-based learning. \looseness=-1

\paragraph{Social Impact}
The social impact of our proposed method is expected to be predominantly positive. Our method can improve drug design pipelines, brain modelling, and modelling complex quantum systems. Those applications will be beneficial to society, and we don't foresee any direct negative impact.

\paragraph{Future Work}
Several directions remain open. Learning heat kernel parameters could improve adaptability across datasets. Extending our method to dynamic combinatorial complexes may enable the study of evolving structures. Finally, combining our approach with contrastive or self-supervised learning could enhance robustness in low-data settings. We expect these steps to strengthen the role of topological deep learning in structured data representation.

\paragraph{Limitations} 
While the neural network training is efficient and fast per iteration, our method requires an expensive preprocessing step to determine a full eigendecomposition of the Laplacian of each combinatorial complex. This still hinders our method from scaling up to combinatorial complexes to a million cells of rank 0.
Furthermore, our method depends on properly selecting diffusion times and network sizes. Parameterising the features could also improve this.

{
    \small
    \bibliographystyle{unsrt}
    \bibliography{main}

\begin{thebibliography}{10}

\bibitem{kipf2016semi}
Thomas~N. Kipf and Max Welling.
\newblock Semi-supervised classification with graph convolutional networks.
\newblock In {\em International Conference on Learning Representations}, 2017.

\bibitem{maron2019provably}
Haggai Maron, Heli Ben-Hamu, Hadar Serviansky, and Yaron Lipman.
\newblock Provably powerful graph networks.
\newblock {\em Advances in neural information processing systems}, 32, 2019.

\bibitem{besta2024demystifying}
Maciej Besta, Florian Scheidl, Lukas Gianinazzi, Grzegorz Kwasniewski, Shachar Klaiman, J{\"u}rgen M{\"u}ller, and Torsten Hoefler.
\newblock Demystifying higher-order graph neural networks.
\newblock {\em arXiv preprint arXiv:2406.12841}, 2024.

\bibitem{hajij2022topological}
Mustafa Hajij, Ghada Zamzmi, Theodore Papamarkou, Nina Miolane, Aldo Guzm{\'a}n-S{\'a}enz, Karthikeyan~Natesan Ramamurthy, Tolga Birdal, Tamal~K Dey, Soham Mukherjee, Shreyas~N Samaga, et~al.
\newblock Topological deep learning: Going beyond graph data.
\newblock {\em arXiv preprint arXiv:2206.00606}, 2022.

\bibitem{verma2024topological}
Yogesh Verma, Amauri~H Souza, and Vikas Garg.
\newblock Topological neural networks go persistent, equivariant, and continuous.
\newblock {\em International Conference on Machine Learning}, 2024.

\bibitem{eitan2024topological}
Yam Eitan, Yoav Gelberg, Guy Bar-Shalom, Fabrizio Frasca, Michael~M. Bronstein, and Haggai Maron.
\newblock Topological blindspots: Understanding and extending topological deep learning through the lens of expressivity.
\newblock In {\em The Thirteenth International Conference on Learning Representations}, 2025.

\bibitem{hajij2020cell}
Mustafa Hajij, Kyle Istvan, and Ghada Zamzmi.
\newblock Cell complex neural networks.
\newblock {\em arXiv preprint arXiv:2010.00743}, 2020.

\bibitem{cai2022convergence}
Chen Cai and Yusu Wang.
\newblock Convergence of invariant graph networks.
\newblock In {\em International Conference on Machine Learning}, pages 2457--2484. PMLR, 2022.

\bibitem{maron2018invariant}
Haggai Maron, Heli Ben-Hamu, Nadav Shamir, and Yaron Lipman.
\newblock Invariant and equivariant graph networks.
\newblock In {\em International Conference on Learning Representations}, 2019.

\bibitem{sun2009concise}
Jian Sun, Maks Ovsjanikov, and Leonidas Guibas.
\newblock A concise and provably informative multi-scale signature based on heat diffusion.
\newblock In {\em Computer graphics forum}, volume~28, pages 1383--1392. Wiley Online Library, 2009.

\bibitem{hoppe2024representing}
Josef Hoppe and Michael~T Schaub.
\newblock Representing edge flows on graphs via sparse cell complexes.
\newblock In {\em Learning on Graphs Conference}, pages 1--1. PMLR, 2024.

\bibitem{hein2007graph}
Matthias Hein, Jean-Yves Audibert, and Ulrike~von Luxburg.
\newblock Graph laplacians and their convergence on random neighborhood graphs.
\newblock {\em Journal of Machine Learning Research}, 8(6), 2007.

\bibitem{ovsjanikov2008global}
Maks Ovsjanikov, Jian Sun, and Leonidas Guibas.
\newblock Global intrinsic symmetries of shapes.
\newblock In {\em Computer graphics forum}, volume~27, pages 1341--1348. Wiley Online Library, 2008.

\bibitem{forman1998witten}
Robin Forman.
\newblock Witten--morse theory for cell complexes.
\newblock {\em Topology}, 37(5):945--979, 1998.

\bibitem{hornik1989multilayer}
Kurt Hornik, Maxwell Stinchcombe, and Halbert White.
\newblock Multilayer feedforward networks are universal approximators.
\newblock {\em Neural networks}, 2(5):359--366, 1989.

\bibitem{yun2019transformers}
Chulhee Yun, Srinadh Bhojanapalli, Ankit~Singh Rawat, Sashank~J Reddi, and Sanjiv Kumar.
\newblock Are transformers universal approximators of sequence-to-sequence functions?
\newblock {\em International Conference on Learning Representations}, 2020.

\bibitem{xu2018powerful}
Keyulu Xu, Weihua Hu, Jure Leskovec, and Stefanie Jegelka.
\newblock How powerful are graph neural networks?
\newblock {\em arXiv preprint arXiv:1810.00826}, 2018.

\bibitem{joeres2024higher}
Roman Joeres and Daniel Bojar.
\newblock Higher-order message passing for glycan representation learning.
\newblock {\em arXiv preprint arXiv:2409.13467}, 2024.

\bibitem{streicher2023graph}
Or~Streicher and Guy Gilboa.
\newblock Graph laplacian for semi-supervised learning.
\newblock In {\em International Conference on Scale Space and Variational Methods in Computer Vision}, pages 250--262. Springer, 2023.

\bibitem{chung1997spectral}
Fan~RK Chung.
\newblock {\em Spectral graph theory}, volume~92.
\newblock American Mathematical Soc., 1997.

\bibitem{ng2001spectral}
Andrew Ng, Michael Jordan, and Yair Weiss.
\newblock On spectral clustering: Analysis and an algorithm.
\newblock {\em Advances in neural information processing systems}, 14, 2001.

\bibitem{belkin2003laplacian}
Mikhail Belkin and Partha Niyogi.
\newblock Laplacian eigenmaps for dimensionality reduction and data representation.
\newblock {\em Neural computation}, 15(6):1373--1396, 2003.

\bibitem{coifman2006diffusion}
Ronald~R Coifman and St{\'e}phane Lafon.
\newblock Diffusion maps.
\newblock {\em Applied and computational harmonic analysis}, 21(1):5--30, 2006.

\bibitem{fu2022p}
Guoji Fu, Peilin Zhao, and Yatao Bian.
\newblock $ p $-laplacian based graph neural networks.
\newblock In {\em International conference on machine learning}, pages 6878--6917. PMLR, 2022.

\bibitem{runwal2022robust}
Bharat Runwal, Sandeep Kumar, et~al.
\newblock Robust graph neural networks using weighted graph laplacian.
\newblock {\em arXiv preprint arXiv:2208.01853}, 2022.

\bibitem{bodnar2021weisfeiler2}
Cristian Bodnar, Fabrizio Frasca, Yuguang Wang, Nina Otter, Guido~F Montufar, Pietro Lio, and Michael Bronstein.
\newblock Weisfeiler and lehman go topological: Message passing simplicial networks.
\newblock In {\em International conference on machine learning}, pages 1026--1037. PMLR, 2021.

\bibitem{battiloro2024n}
Claudio Battiloro, Ege Karaismailoglu, Mauricio Tec, George Dasoulas, Michelle Audirac, and Francesca Dominici.
\newblock E(n) equivariant topological neural networks.
\newblock In {\em The Thirteenth International Conference on Learning Representations}, 2025.

\bibitem{roddenberry2021principled}
T~Mitchell Roddenberry, Nicholas Glaze, and Santiago Segarra.
\newblock Principled simplicial neural networks for trajectory prediction.
\newblock In {\em International Conference on Machine Learning}, pages 9020--9029. PMLR, 2021.

\bibitem{yang2023convolutional}
Maosheng Yang and Elvin Isufi.
\newblock Convolutional learning on simplicial complexes.
\newblock {\em arXiv preprint arXiv:2301.11163}, 2023.

\bibitem{ebli2020simplicial}
Stefania Ebli, Micha{\"e}l Defferrard, and Gard Spreemann.
\newblock Simplicial neural networks.
\newblock {\em arXiv preprint arXiv:2010.03633}, 2020.

\bibitem{goh2022simplicial}
Christopher Wei~Jin Goh, Cristian Bodnar, and Pietro Lio.
\newblock Simplicial attention networks.
\newblock {\em arXiv preprint arXiv:2204.09455}, 2022.

\bibitem{battiloro2024generalized}
Claudio Battiloro, Lucia Testa, Lorenzo Giusti, Stefania Sardellitti, Paolo Di~Lorenzo, and Sergio Barbarossa.
\newblock Generalized simplicial attention neural networks.
\newblock {\em IEEE Transactions on Signal and Information Processing over Networks}, 2024.

\bibitem{maggs2023simplicial}
Kelly Maggs, Celia Hacker, and Bastian Rieck.
\newblock Simplicial representation learning with neural $ k $-forms.
\newblock {\em arXiv preprint arXiv:2312.08515}, 2023.

\bibitem{lecha2025higher}
Manuel Lecha, Andrea Cavallo, Francesca Dominici, Elvin Isufi, and Claudio Battiloro.
\newblock Higher-order topological directionality and directed simplicial neural networks.
\newblock In {\em ICASSP 2025-2025 IEEE International Conference on Acoustics, Speech and Signal Processing (ICASSP)}, pages 1--5. IEEE, 2025.

\bibitem{bodnar2021weisfeiler}
Cristian Bodnar, Fabrizio Frasca, Nina Otter, Yuguang Wang, Pietro Lio, Guido~F Montufar, and Michael Bronstein.
\newblock Weisfeiler and lehman go cellular: Cw networks.
\newblock {\em Advances in neural information processing systems}, 34:2625--2640, 2021.

\bibitem{hajij2022higher}
Mustafa Hajij, Ghada Zamzmi, Theodore Papamarkou, Nina Miolane, Aldo Guzm{\'a}n-S{\'a}enz, and Karthikeyan~Natesan Ramamurthy.
\newblock Higher-order attention networks.
\newblock {\em arXiv preprint arXiv:2206.00606}, 2(3):4, 2022.

\bibitem{papillon2023architectures}
Mathilde Papillon, Sophia Sanborn, Mustafa Hajij, and Nina Miolane.
\newblock Architectures of topological deep learning: A survey of message-passing topological neural networks.
\newblock {\em arXiv preprint arXiv:2304.10031}, 2023.

\bibitem{horn2021topological}
Max Horn, Edward De~Brouwer, Michael Moor, Yves Moreau, Bastian Rieck, and Karsten Borgwardt.
\newblock Topological graph neural networks.
\newblock {\em arXiv preprint arXiv:2102.07835}, 2021.

\bibitem{chen2021topological}
Yuzhou Chen, Baris Coskunuzer, and Yulia Gel.
\newblock Topological relational learning on graphs.
\newblock {\em Advances in neural information processing systems}, 34:27029--27042, 2021.

\bibitem{ballester2023expressivity}
Rub{\'e}n Ballester and Bastian Rieck.
\newblock On the expressivity of persistent homology in graph learning.
\newblock {\em arXiv preprint arXiv:2302.09826}, 2023.

\bibitem{buffelli2024cliqueph}
Davide Buffelli, Farzin Soleymani, and Bastian Rieck.
\newblock Cliqueph: Higher-order information for graph neural networks through persistent homology on clique graphs.
\newblock {\em arXiv preprint arXiv:2409.08217}, 2024.

\bibitem{telyatnikov2024topobenchmarkx}
Lev Telyatnikov, Guillermo Bernárdez, Marco Montagna, Pavlo Vasylenko, Ghada Zamzmi, Mustafa Hajij, Michael~T. Schaub, Nina Miolane, Simone Scardapane, and Theodore Papamarkou.
\newblock Topobenchmarkx: A framework for benchmarking topological deep learning.
\newblock {\em CoRR}, abs/2406.06642, 2024.

\bibitem{papillon2024topotune}
Mathilde Papillon, Guillermo Bern{\'a}rdez, Claudio Battiloro, and Nina Miolane.
\newblock Topotune: A framework for generalized combinatorial complex neural networks.
\newblock {\em arXiv preprint arXiv:2410.06530}, 2024.

\bibitem{morris2019weisfeiler}
Christopher Morris, Martin Ritzert, Matthias Fey, William~L Hamilton, Jan~Eric Lenssen, Gaurav Rattan, and Martin Grohe.
\newblock Weisfeiler and leman go neural: Higher-order graph neural networks.
\newblock In {\em Proceedings of the AAAI conference on artificial intelligence}, volume~33, pages 4602--4609, 2019.

\bibitem{weisfeiler1968reduction}
Boris Weisfeiler and Andrei Leman.
\newblock The reduction of a graph to canonical form and the algebra which appears therein.
\newblock {\em nti, Series}, 2(9):12--16, 1968.

\bibitem{abboud2020surprising}
Ralph Abboud, Ismail~Ilkan Ceylan, Martin Grohe, and Thomas Lukasiewicz.
\newblock The surprising power of graph neural networks with random node initialization.
\newblock {\em arXiv preprint arXiv:2010.01179}, 2020.

\bibitem{bouritsas2022improving}
Giorgos Bouritsas, Fabrizio Frasca, Stefanos Zafeiriou, and Michael~M Bronstein.
\newblock Improving graph neural network expressivity via subgraph isomorphism counting.
\newblock {\em IEEE Transactions on Pattern Analysis and Machine Intelligence}, 45(1):657--668, 2022.

\bibitem{puny2023expressive}
Omri Puny, Derek Lim, Bobak Kiani, Haggai Maron, and Yaron Lipman.
\newblock Equivariant polynomials for graph neural networks.
\newblock In {\em International Conference on Machine Learning}, pages 28191--28222. PMLR, 2023.

\bibitem{bevilacqua2021equivariant}
Beatrice Bevilacqua, Fabrizio Frasca, Derek Lim, Balasubramaniam Srinivasan, Chen Cai, Gopinath Balamurugan, Michael~M Bronstein, and Haggai Maron.
\newblock Equivariant subgraph aggregation networks.
\newblock {\em International Conference on Learning Representations}, 2022.

\bibitem{frasca2022understanding}
Fabrizio Frasca, Beatrice Bevilacqua, Michael Bronstein, and Haggai Maron.
\newblock Understanding and extending subgraph gnns by rethinking their symmetries.
\newblock {\em Advances in Neural Information Processing Systems}, 35:31376--31390, 2022.

\bibitem{zhang2023complete}
Bohang Zhang, Guhao Feng, Yiheng Du, Di~He, and Liwei Wang.
\newblock A complete expressiveness hierarchy for subgraph gnns via subgraph weisfeiler-lehman tests.
\newblock In {\em International Conference on Machine Learning}, pages 41019--41077. PMLR, 2023.

\bibitem{zhang2021nested}
Muhan Zhang and Pan Li.
\newblock Nested graph neural networks.
\newblock {\em Advances in Neural Information Processing Systems}, 34:15734--15747, 2021.

\bibitem{cotta2021reconstruction}
Leonardo Cotta, Christopher Morris, and Bruno Ribeiro.
\newblock Reconstruction for powerful graph representations.
\newblock {\em Advances in Neural Information Processing Systems}, 34:1713--1726, 2021.

\bibitem{bar2024flexible}
Guy Bar-Shalom, Yam Eitan, Fabrizio Frasca, and Haggai Maron.
\newblock A flexible, equivariant framework for subgraph {GNN}s via graph products and graph coarsening.
\newblock In {\em The Thirty-eighth Annual Conference on Neural Information Processing Systems}, 2024.

\bibitem{bamberger2022topological}
Jacob Bamberger.
\newblock A topological characterisation of weisfeiler-leman equivalence classes.
\newblock In {\em Topological, Algebraic and Geometric Learning Workshops 2022}, pages 17--27. PMLR, 2022.

\bibitem{jegelka2022theory}
Stefanie Jegelka.
\newblock Theory of graph neural networks: Representation and learning.
\newblock In {\em The International Congress of Mathematicians}, pages 1--23, 2022.

\bibitem{morris2023weisfeiler}
Christopher Morris, Yaron Lipman, Haggai Maron, Bastian Rieck, Nils~M Kriege, Martin Grohe, Matthias Fey, and Karsten Borgwardt.
\newblock Weisfeiler and leman go machine learning: The story so far.
\newblock {\em Journal of Machine Learning Research}, 24(333):1--59, 2023.

\bibitem{zhang2024expressive}
Bingxu Zhang, Changjun Fan, Shixuan Liu, Kuihua Huang, Xiang Zhao, Jincai Huang, and Zhong Liu.
\newblock The expressive power of graph neural networks: A survey.
\newblock {\em IEEE Transactions on Knowledge and Data Engineering}, 2024.

\bibitem{bronstein2010scale}
Michael~M Bronstein and Iasonas Kokkinos.
\newblock Scale-invariant heat kernel signatures for non-rigid shape recognition.
\newblock In {\em 2010 IEEE computer society conference on computer vision and pattern recognition}, pages 1704--1711. IEEE, 2010.

\bibitem{raviv2010volumetric}
Dan Raviv, Michael~M Bronstein, Alexander~M Bronstein, and Ron Kimmel.
\newblock Volumetric heat kernel signatures.
\newblock In {\em Proceedings of the ACM workshop on 3D object retrieval}, pages 39--44, 2010.

\bibitem{ovsjanikov2010one}
Maks Ovsjanikov, Quentin M{\'e}rigot, Facundo M{\'e}moli, and Leonidas Guibas.
\newblock One point isometric matching with the heat kernel.
\newblock In {\em Computer Graphics Forum}, volume~29, pages 1555--1564. Wiley Online Library, 2010.

\bibitem{liu2022novel}
Yantao Liu, Luca Rossi, and Andrea Torsello.
\newblock A novel graph kernel based on the wasserstein distance and spectral signatures.
\newblock In {\em Joint IAPR International Workshops on Statistical Techniques in Pattern Recognition (SPR) and Structural and Syntactic Pattern Recognition (SSPR)}, pages 122--131. Springer, 2022.

\bibitem{aubry2011wave}
Mathieu Aubry, Ulrich Schlickewei, and Daniel Cremers.
\newblock The wave kernel signature: A quantum mechanical approach to shape analysis.
\newblock In {\em 2011 IEEE international conference on computer vision workshops (ICCV workshops)}, pages 1626--1633. IEEE, 2011.

\bibitem{donnat2018learning}
Claire Donnat, Marinka Zitnik, David Hallac, and Jure Leskovec.
\newblock Learning structural node embeddings via diffusion wavelets.
\newblock In {\em Proceedings of the 24th ACM SIGKDD international conference on knowledge discovery \& data mining}, pages 1320--1329, 2018.

\bibitem{keros2023spectral}
Alexandros Keros and Kartic Subr.
\newblock Spectral coarsening with hodge laplacians.
\newblock In {\em ACM SIGGRAPH 2023 Conference Proceedings}, pages 1--11, 2023.

\bibitem{merris1994laplacian}
Russell Merris.
\newblock Laplacian matrices of graphs: a survey.
\newblock {\em Linear algebra and its applications}, 197:143--176, 1994.

\end{thebibliography}
}
\newpage

\section{Appendix}
In this part of the Appendix, we fully describe the proofs and provide further definitions if needed. We also include the full text for completeness and ease of reading. 

\subsection{Further definitions}

\begin{definition}
Let \( G = (V, E) \) be an undirected graph with \( n = |V| \) nodes and \( m = |E| \) edges. The \emph{incidence matrix} \( I \in \mathbb{R}^{n \times m} \) of \( G \) is defined as follows:

For each edge \( e_k = (i, j) \in E \), assign an arbitrary orientation (e.g., from node \( i \) to node \( j \)). Then the \( k \)-th column of \( I \) is given by:
\[
I_{v,k} =
\begin{cases}
+1 & \text{if } v = i \text{ (source of edge)} \\
-1 & \text{if } v = j \text{ (target of edge)} \\
0 & \text{otherwise}
\end{cases}
\]

\end{definition}

\paragraph{Simplicial Complex}
A \textbf{simplicial complex} \( K \) is a collection of subsets (called \emph{simplices}) formed from a finite set \( V \) of vertices. Each vertex \( v \in V \) appears in \( K \) as a singleton set \( \{v\} \), and any higher-dimensional simplex \( \sigma = \{v_0, \dots, v_k\} \subset V \) represents a \textbf{\( k \)-simplex}, where the dimension is \( k = |\sigma| - 1 \). Examples include:
\begin{itemize}
    \item \textbf{0-simplices}: vertices
    \item \textbf{1-simplices}: edges
    \item \textbf{2-simplices}: triangles
    \item \textbf{3-simplices}: tetrahedra
\end{itemize}

A key property of a simplicial complex is that \emph{every subset} \( \tau \subset \sigma \) of a simplex \( \sigma \in K \) must also be included in \( K \). The \textbf{dimension} of the complex is the highest dimension among its simplices. Also, each simplex only contains simplices of one lower dimension. For example:
\begin{itemize}
    \item A graph is a 1-dimensional simplicial complex.
    \item A triangle mesh is a 2-dimensional complex.
\end{itemize}

\paragraph{Boundary Operator}

A \textbf{boundary operator} \( d_k \) is a linear map that captures how \( k \)-simplices in a simplicial complex are bounded by \((k-1)\)-simplices. It generalises the concept of incidence matrices from graphs to higher dimensions. Specifically,
\[
d_k : \mathbb{R}[K_k] \rightarrow \mathbb{R}[K_{k-1}]
\]

maps each \( k \)-simplex to a formal sum of its \((k - 1)\)-dimensional faces.

To define \( d_k \), the vertices in \( V \) are ordered, and a \( k \)-simplex is expressed as an ordered list \( \sigma = [v_0, \dots, v_k] \). Then:
\[
d_k(\sigma) = \sum_{i=0}^k (-1)^i \sigma_{-i}
\]
where \( \sigma_{-i} \) denotes the \((k - 1)\)-simplex obtained by removing the \( i \)-th vertex from \( \sigma \). Based on the vertex ordering, the signs encode \textbf{orientation}.

The boundary operator reflects how each simplex connects to its lower-dimensional components and is a core concept in \emph{algebraic topology} and \emph{discrete differential geometry}. Those definitions align with the definition from \cite{keros2023spectral}.

\subsection{Laplacian Properties}

\begin{corollary} \textbf{Relationship between CC and Graph Laplacians}
For graphs, combinatorial complexes with rank 1, the Laplacian from Def. 3.1 is identical to the Laplacian defined on graphs.
\end{corollary}

\begin{proof}
We show that the combinatorial complex Laplacian \( L_C \), when restricted to rank-0 cells and using only rank-1 adjacency, coincides with the standard graph Laplacian \( L_G \).

Let \( G = (V, E) \) be an undirected graph. Its incidence matrix \( I_G \in \mathbb{R}^{|V| \times |E|} \) assigns \( +1 \) and \( -1 \) to the source and target nodes of each edge, respectively, under an arbitrary orientation. It is well known (see \cite{merris1994laplacian}) that the graph Laplacian satisfies:
\[
L_G = I_G I_G^\top.
\]

Now consider a combinatorial complex \( \mathcal{C} \) consisting only of rank-0 and rank-1 cells, where rank-0 cells correspond to graph vertices and rank-1 cells to edges. The boundary operator \( \delta_0 \) from rank-1 to rank-0 is then equivalent to \( I_G \), up to sign convention.

Let \( L_C = \delta_0 \delta_0^\top \) denote the Laplacian on rank-0 cells of \( \mathcal{C} \). Then:
\[
L_C = \delta_0 \delta_0^\top = I_G I_G^\top = L_G.
\]

Thus, the combinatorial complex Laplacian reduces to the standard graph Laplacian in the rank-0/1 case.
\end{proof}

\begin{theorem}\textbf{Uniqueness of Laplacians for CCs}
\label{them:uniqueness_of_laplacians}
Let \( L \) be the Laplacian of a combinatorial complex \( C \). The Laplacian of a combinatorial complex is uniquely determined. Moreover, if there exists an invertible orthogonal matrix \( \mathbf{\Pi} \in \mathbb{R}^{n \times n} \) such that  
\begin{align}
      L' = \mathbf{\Pi} L \mathbf{\Pi}^{\top},
\end{align}
  
then \( L' \) serves as the Laplacian of another combinatorial complex \( C' \), which is spectrally equivalent to \( C \), meaning there exists a bijective, unique mapping between the two combinatorial complexes, which makes them isomorphic.
\end{theorem}

\begin{proof}
Let \( \mathcal{C} \) be a fixed combinatorial complex with a unique and fixed set of cells. Since the complex structure is fixed, the incidence matrices \( \delta_k \) between cells of adjacent ranks are uniquely determined (up to orientation and indexing). Therefore, each term \( \delta_k^\top \delta_k \) and \( \delta_k \delta_k^\top \) used in constructing Laplacian operators is fixed for the given complex.

Consider the construction of a Laplacian acting on \( 0 \)-cells:
\[
L = \sum_r \delta_r\delta_r^\top
\]

If \( B \) denotes  \( \delta_k \), then \( \Delta_k \) involves terms of the form \( BB^\top \), which define symmetric positive semi-definite matrices. If another matrix \( B' \) satisfies \( B'B'^\top = BB^\top \), then it must hold that \( B' = BQ \) for some orthogonal matrix \( Q \in \mathbb{R}^{r \times r} \), assuming \( B \) has full rank.

Thus, any such factorisation is unique up to an orthogonal transformation. Furthermore, all the individual summands are uniquely weighted. Because the combinatorial structure fixes the incidence relations, the overall Laplacian operator \( \Delta_k \) is determined uniquely by a basis transformation on intermediate rank cells.

Furthermore, since the Laplacian acts on fixed \( k \)-cells, and any ambiguity from orientation or ordering of higher-rank cells affects all terms consistently, the resulting operator \( \Delta_k \) is unique up to a consistent transformation (e.g., permutation or rotation), which does not affect its spectrum.

Therefore, the Laplacian \( L \) for a fixed combinatorial complex is unique up to orthogonal equivalence.
\end{proof}

\begin{lemma}\textbf{Non-uniqueness of Hodge Laplacians on CC} 
\label{lem:non_unique}Hodge Laplacians on Combinatorial complexes are not unique, meaning there exists a pair of Combinatorial Complexes which share a Hodge Laplacian, but are not isomorphic
\end{lemma}

\begin{figure}
\centering
        \includegraphics[width=0.5\linewidth]{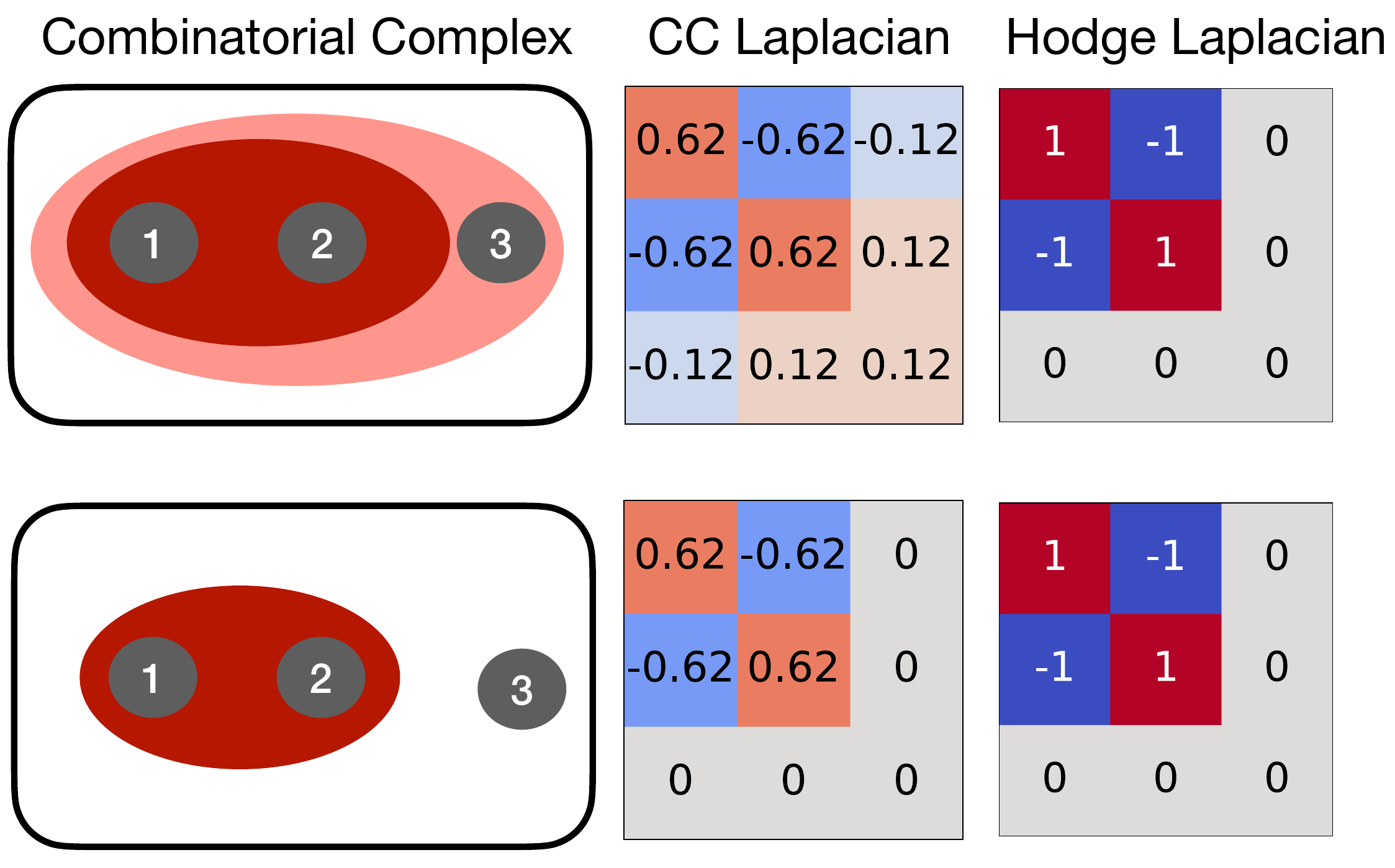}
    \caption{Presenting two Combinatorial Complexes with their CC and Hodge Laplacian. While the CC Laplacian differs, the Hodge Laplacian is the same for both complexes.}
    \label{fig:counter_example_hodge}
\end{figure}
\begin{proof}

We prove the corollary by providing a counterexample, as illustrated in \cref{fig:counter_example_hodge}. The figure depicts two non-isomorphic combinatorial complexes. In both complexes, cells 1, 2, and 3 are of rank 0, and cells 1 and 2 are connected via a rank-1 cell. However, cell 3 is additionally connected to the rest of the structure in the first complex through a higher-order cell of rank 4.

This higher-order connection introduces a structural difference that breaks isomorphism between the two complexes. The \emph{Combinatorial Complex (CC) Laplacian}, defined in Definition 3.1, captures this distinction by incorporating interactions across all ranks. Specifically, it reflects the influence of the rank-4 cell, which connects otherwise disconnected components at rank 0.

In contrast, the \emph{Hodge Laplacian} fails to distinguish the two complexes. Since the rank-4 cell is not incident to any rank-3 or rank-5 cell, its contribution to the Hodge Laplacian vanishes (as it produces zero under boundary and coboundary operators). Consequently, the Hodge Laplacians of both complexes are identical.

This example demonstrates that the Hodge Laplacian is not a unique or complete descriptor of combinatorial complex structure. In contrast, the CC Laplacian distinguishes between them, establishing its greater expressiveness and discriminative power.
\end{proof}

\begin{corollary}
\label{thm:hodge_expressiveness}
\textbf{Hodge Laplacian Expressiveness}    On Combinatorial Complexes, the Laplacian in Definition 3.1 is strictly more expressive than the Hodge Laplacian, and on Simplicial Complexes they are equally expressive.
\end{corollary}
\begin{proof}
  We have shown the first part of the proof in \cref{lem:non_unique}.

  We now show that the CC Laplacian and the Hodge Laplacian uniquely capture structural differences in cell complexes.

This follows directly from the CC Laplacian's construction: as shown earlier, the Laplacian is uniquely determined by the complex's combinatorial structure. Since the incidence relations between cells are fixed, the CC Laplacian is uniquely defined for any combinatorial complex.

Similarly, the Hodge Laplacian is uniquely defined on cell complexes. First, we observe that the coboundary operator \( d_k \), which maps \( k \)-cochains to \( (k+1) \)-cochains, is uniquely determined by the cell structure and chosen orientation. Given this, the Hodge Laplacian,
\[
\Delta_k = d_{k-1}^\top d_{k-1} + d_k d_k^\top,
\]
is also uniquely defined for each \( k \).

Thus, both Laplacians yield unique operators for any fixed cell complex structure. This completes the proof.

\end{proof}

\paragraph{Laplacian interpretations} 
\begin{theorem}
\textbf{Smoothness}
    Let $L$ be the Laplacian descriptor for the Combinatorial Complex $C$. We then define a function $f: C^0 \rightarrow \mathbb{R}$. Based on this:
\begin{align}
    f^\top L f
    \label{eq:smoothness}
\end{align}
expresses the smoothness of a function defined on rank-0 cells of the combinatorial complex. In particular, \cref{eq:smoothness} can be reformulated as follows, where \( B = \sum_{i=0}^R \beta_i \):
\begin{align}
    B\sum_{i,j} w_{i,j}(f_i - f_j)^2,
    \label{eq:smoothness_itemised}
\end{align}
\end{theorem}

\begin{proof}
We begin by expanding the quadratic form:
\[
f^\top L f = \sum_{i=1}^n \sum_{j=1}^n f_i L_{ij} f_j
\]

By the Laplacian definition:
\[
L_{ij} = 
\begin{cases}
\sum_{k \neq i} w_{ik} & \text{if } i = j \\
-w_{ij} & \text{if } i \neq j
\end{cases}
\]

Substituting, we get:
\[
f^\top L f = \sum_{i=1}^n f_i \left( \sum_{j \neq i} (-w_{ij}) f_j + L_{ii} f_i \right)
= \sum_{i=1}^n \left( L_{ii} f_i^2 - \sum_{j \neq i} w_{ij} f_i f_j \right)
\]

Now, using symmetry \( w_{ij} = w_{ji} \), we can symmetrize:
\[
f^\top L f = \sum_{i < j} w_{ij}(f_i^2 + f_j^2 - 2f_i f_j)
= \sum_{i < j} w_{ij}(f_i - f_j)^2
\]

This concludes the proof.
\end{proof}

\begin{theorem} 
\textbf{HKS uniqueness}
    Let \( L \) and \( L' \) be two Laplacians such that \( L' \neq \mathbf{\Pi} L \mathbf{\Pi}^{\top} \) for any orthogonal matrix \( \mathbf{\Pi} \). Then the corresponding Heat Kernel Signature (HKS) descriptors derived from \( L \) and \( L' \) are distinct.
\end{theorem}
\begin{proof}
We assume the uniqueness of the Laplacian \( L \), as established in the previous theorem. It remains to show that the corresponding diffusion kernel is uniquely determined by the spectrum of \( L \).

Let \( L = \Phi \Lambda \Phi^\top \) be the eigendecomposition of the symmetric Laplacian, where \( \Phi \in \mathbb{R}^{n \times n} \) is an orthonormal matrix of eigenvectors and \( \Lambda = \mathrm{diag}(\lambda_1, \dots, \lambda_n) \) is the diagonal matrix of eigenvalues. The heat diffusion kernel at time \( t > 0 \) is defined as:

\[
K_t := \Phi \, \mathrm{diag}(e^{-t\lambda_1}, \dots, e^{-t\lambda_n}) \, \Phi^\top
\]

We aim to show that this kernel is unique for a fixed Laplacian. First, note that the exponential function \( x \mapsto e^{-tx} \) is strictly decreasing and injective on \( \mathbb{R} \). Therefore, the map \( \lambda_i \mapsto e^{-t\lambda_i} \) preserves uniqueness of the spectrum.

Since the eigenvectors \( \Phi \) are also uniquely determined up to orthogonal transformations (and these cancel in the product \( \Phi \Phi^\top \)), the matrix \( K_t \) is uniquely determined by \( L \).

Thus, the diffusion kernel \( K_t \) is uniquely defined for a given Laplacian and a fixed diffusion time \( t \). If the same kernel were to arise from two distinct spectra \( \Lambda \neq \Lambda' \), then we would obtain \( e^{-t\lambda_i} = e^{-t\lambda_i'} \) for some \( i \), contradicting the injectivity of the exponential map.

Hence, the diffusion pattern is uniquely determined, completing the proof.
\end{proof}

\begin{corollary} \textbf{Expressiveness}
Given two combinatorial complexes with distinct input descriptors, it is possible to learn a function using a Universal Function Approximator (UFA) approach that effectively distinguishes between them. This means we can determine any WL classes by theoretical design. \looseness=-1
\end{corollary}

\begin{proof}
Let \( \mathcal{C}_1 \) and \( \mathcal{C}_2 \) be two combinatorial complexes. Assume that their node-level input features (e.g., heat kernel signatures) are such that \( \mathcal{C}_1 \not\cong \mathcal{C}_2 \Rightarrow X_1 \not\equiv X_2 \), i.e., the inputs are distinctive up to isomorphism.

Let \( f_\theta \) be a neural network modelled as a Universal Function Approximator (UFA), which takes the input \( X \) and computes an output \( f_\theta(X) \). Since UFAs can approximate any continuous function to arbitrary precision, there exists a parameterisation \( \theta \) such that:
\[
f_\theta(X_1) \neq f_\theta(X_2) \quad \text{whenever } X_1 \not\equiv X_2
\]

Hence, as long as the inputs are distinctive concerning isomorphism, the network can be trained to produce unique outputs for each non-isomorphic complex. This implies that the method is expressive enough to distinguish between combinatorial complexes up to isomorphism.
\end{proof}

\end{document}